\documentclass[12pt]{report}

\usepackage[a-1b]{pdfx}

\usepackage{pdfpages}
\usepackage{pax}

\usepackage[american]{babel}

\usepackage[T1]{fontenc}
\usepackage[utf8]{inputenc}

\usepackage{lmodern}

\usepackage{lipsum}

\usepackage{textcomp}
\usepackage{amsmath}
\usepackage{amsfonts}
\usepackage{amsthm}
\usepackage{commath}

\usepackage{soul}

\usepackage{geometry}
\usepackage{setspace}

\usepackage[protrusion]{microtype}
\usepackage{multicol}
\usepackage{seqsplit}

\usepackage{listings}
\usepackage[pdfa]{hyperref}
\hypersetup{
   linktocpage,
   unicode,
   colorlinks=true,
   citecolor=link,
   filecolor=link,
   linkcolor=link,
   urlcolor=link
}

\usepackage{xcolor}

\definecolor{link}{rgb}{0.45,0.51,0.67}

\usepackage[titles]{tocloft}

\usepackage{tocbibind}

\usepackage{calc}
\renewcommand{\cftchappresnum}{\chaptername\space}
\newcommand\chap[1]{%
  \chapter*{#1}%
  \addcontentsline{toc}{chapter}{#1}}

\usepackage{booktabs}
\usepackage{multirow}
\usepackage{longtable}
\usepackage{array}
\usepackage{pdflscape}

\usepackage{enumitem}

\usepackage{graphicx}

\usepackage[font=small,
            labelfont=bf,
            labelsep=period,
            font=sf,
            hypcap=true]{caption}

\usepackage{ifthen}
\usepackage[rightcaption]{sidecap}

\usepackage[all]{hypcap}

\graphicspath{{./images/}}

\usepackage{appendix}

\usepackage{cite}
\usepackage[numbers,square,compress]{natbib}
\usepackage{chbibref}

\usepackage{amsmath,amssymb,array}

\newtheorem{theorem}{Theorem}[section]

\begin{document}
\setlength{\parskip}{3pt}
\baselineskip=24pt
%
\pagenumbering{roman}
%
\thispagestyle{empty}
\baselineskip=18pt
\begin{center}
\vspace*{3\baselineskip}
%
{\bfseries Malicious Network Traffic Detection via Deep Learning: \\An Information Theoretic View}\\[6\baselineskip]
by\\
%
Erick Galinkin\\[3\baselineskip]
A thesis submitted to The Johns Hopkins University in conformity\\
with the requirements for the degree of \\
Master of Science in Applied and Computational Mathematics\\[4\baselineskip]
Baltimore, Maryland\\
August, 2020\\[6\baselineskip]
%
{\copyright{} 2020 E.~Galinkin\\
All rights reserved}
\end{center}
%
\baselineskip=24pt
\newpage 

%
\pagestyle{plain}
\setcounter{page}{2}
%
\chapter*{Abstract}
\addcontentsline{toc}{chapter}{Abstract}
The attention that deep learning has garnered from the academic community and industry continues to grow year over year, and it has been said that we are in a new golden age of artificial intelligence research.
However, neural networks are still often seen as a ``black box'' where learning occurs but cannot be understood in a human-interpretable way.
Since these machine learning systems are increasingly being adopted in security contexts, it is important to explore these interpretations.
We consider an Android malware traffic dataset for approaching this problem.
Then, using the information plane, we explore how homeomorphism affects learned representation of the data and the invariance of the mutual information captured by the parameters on that data.
We empirically validate these results, using accuracy as a second measure of similarity of learned representations.

Our results suggest that although the details of learned representations and the specific coordinate system defined over the manifold of all parameters differ slightly, the functional approximations are the same.
Furthermore, our results show that since mutual information remains invariant under homeomorphism, only feature engineering methods that alter the entropy of the dataset will change the outcome of the neural network.
This means that for some datasets and tasks, neural networks require meaningful, human-driven feature engineering or changes in architecture to provide enough information for the neural network to generate a sufficient statistic.
Applying our results can serve to guide analysis methods for machine learning engineers and suggests that neural networks that can exploit the convolution theorem are equally accurate as standard convolutional neural networks, and can be more computationally efficient.
%
\section*{Thesis Readers}
\begin{singlespace}
\noindent Dr.~Cleon Davis (Primary Advisor)\\
\indent \indent Senior Professional Staff\\
\indent \indent Johns Hopkins University Applied Physics Laboratory\\

\indent \indent and\\

\indent \indent Program Vice Chair Electrical and Computer Engineering\\
\indent \indent Lecturer and Research Faculty in Applied and Computational Mathematics\\
\indent \indent Johns Hopkins Engineering for Professionals\\

\noindent Dr.~Lanier Watkins\\
\indent \indent Program Chair\\
\indent \indent Department of Computer Science and Cybersecurity\\
\indent \indent Johns Hopkins University\\
\end{singlespace}
%
%
%
%
\chapter*{Acknowledgements}
\addcontentsline{toc}{chapter}{Acknowledgements}

I would like to extend my deepest gratitude to my thesis adviser, Dr. Cleon Davis, and my co-adviser Dr. Lanier Watkins.
I have learned so much in this process, and you both have been incredibly supportive. 
I hope that in the future I will make you both proud to call me a peer.

I would also like to extend my gratitude to both Dr. Raymond Canzanese of Netskope, and Mr. Derek Abdine, my two work supervisors during my time at Johns Hopkins.
You have both always been patient, understanding, and flexible.

A debt of gratitude is also owed to Ariel Herbert-Voss of OpenAI and Harvard University, and one of the most inspiring hackers I know for her contribution to the mutual information portion of this thesis.

I would like to honor my colleague and friend, Stella Biderman, for the time she spent hearing my ideas and helping me contextualize them. 

My undying gratitude also goes to Will Pearce, who read through every line of this thesis when it was not quite ready for prime time and helped me get it into a form which I can be proud of.

Lastly, thanks my friends Gloria Cho; Dave Roberts; the entire Rupeethon crew (Jon, Ian, Rich, Conor, Lynn, Scott, Steve, Koios); and my friends in the DEF CON AI village (Sven, Rich, Yaga, Rob, Dr. deltazero, Jason, and all the others). 
No matter how galaxy-brained I got, you always listened and kept me inspired.
Thank you all so much.
%
%
\tableofcontents
%
%
\listoffigures
\clearpage
%
%
\pagenumbering{arabic}
%
\chapter{Introduction}
\label{chap:intro}

Malicious software (malware), has long been a burden on users of computers and the internet. 
For individuals, a malware attack can cause the loss of their personal data and may allow hackers to access their bank accounts, steal their identity, or hold all of the files on their computer ransom. 
For a business, the effects can be extremely dire, with the average cost of a malware attack sitting around \$1.7 million~\cite{seals2019threatlist}. 
This problem is continuing to grow, and malware authors innovate in an effort to circumvent existing mitigating controls. 
As the IRA said to Margaret Thatcher after the Brighton Hotel Bombing: ``We only have to be lucky once. You will have to be lucky always''~\cite{thomas1984this}. 
So too is the case for malware, which must only find one vulnerable system to cause immense damage, while defenders must be lucky on all of their systems. 
As such, traditional signature-based antivirus engines have become less effective against unseen strains of malware~\cite{oconnor2017how}, and so machine learning techniques have helped provide detection of these threats on the endpoint.

In the case of mobile malware, we often do not have the luxury of running computationally intensive processes on an endpoint, nor do we have the ability to remove malware from devices that we do not own that are brought onto our networks. 
As such, we need to leverage this same machine learning technology to perform endpoint-agnostic network detection of malware threats. 
In these cases, a network-based solution seems ideal, as this allows us to mitigate vectors commonly used by modern worms~\cite{mohurle2017brief}. 
Traditional intrusion detection systems like Snort will continue to mitigate known threats, but these signature-based systems suffer from the same curse of reactivity as the traditional antivirus. 
In order to mitigate these threats, mitigating controls must be learned on the fly instead of reactively post-hoc by a human analyst.

Machine learning is a form of statistical learning that emphasizes predictions and pattern recognition in data~\cite{james14introduction}.
Generally, machine learning is viewed as a subset of artificial intelligence where algorithms build mathematical models based on sample data in order to make predictions without being programmed to do so~\cite{bishop2006pattern}.
Machine learning is broadly divided into 3 categories: 
Supervised learning, where a function to map data to a set of associated labels is learned;
Unsupervised learning, where groupings or clusters of the data are identified without labels;
and Reinforcement learning, which is concerned with experiential learning through agent interactions with an environment.
Though there are more granular categories, these suffice for most purposes.
In our case, we are concerned with a classification task for data where we do have labels - so we concern ourselves with supervised learning throughout this work. 

\section{Neural Networks}\label{intro:nn}
Neural Networks are a type of connectionist machine learning system in which artificial neurons are connected to one another in an attempt to emulate biological cognitive functions.
Each neuron is a node that connects to others in a way that mimics the dendrite-synapse-axon connections, illustrated in Figure~\ref{fig:neuron}.
Each of these connections, similar to the myelin sheath in the brain, has a weight that determines the strength of any node on another.

Deep learning~\cite{goodfellow2016deep} is a particular form of machine learning in which artificial neurons are stacked on top of one another in two or more layers. 
Having a larger number of layers allows the network to learn increasingly complex representations at the cost of training speed and the need for more data. 
This has reasonably led us to ask what the applicability of neural networks and deep learning are to domains other than natural language processing~\cite{goldberg2016primer} and computer vision~\cite{lecun1998gradient}.

\begin{figure}
\includegraphics{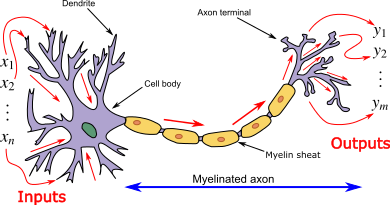}
\centering
\caption{Neuron and mylinated axon with signal flow by Egm4313.s12 (Prof. Loc Vu-Quoc) - Own work, CC BY-SA 4.0, \url{https://commons.wikimedia.org/w/index.php?curid=72816083}}
\label{fig:neuron}
\end{figure}

Neural networks are not new and are quite closely related to the work of Gauss and Legendre~\cite{calin2020deep} on polynomial regression.
This linear approximator takes the weights of the neuron $w_i$ and multiplies them by the input $x_i$ and compares them with a threshold (or bias, as we will refer to it), $b$. 
So for our linear approximator, $L$, we have: $L(x) = \sum_{i=1}^{n}w_i x_i - b$.
In order to determine whether or not this neuron ``fires'', we use a nonlinear activation function such as the Heaviside function, the sigmoid function, or the rectified linear unit (ReLU), which we denote $\sigma$. 
First, define the Heaviside function, which we denote $\delta(x)$~\footnote{The Heaviside function is typically denoted $H(x)$ but we reserve that notation for entropy, which we define in Chapter~\ref{chap:two}}:
\begin{equation}
	\delta(x) = 
	\begin{cases}
		0 \text{ if } x < 0\\
		1 \text{ if } x \geq 0\\
	\end{cases}
\end{equation}
We define the Rectified Linear Unit (ReLU) by $\sigma: \mathbb{R} \to \mathbb{R}$: 
\begin{equation}
	\sigma(z) = z \delta(z) = max\{0, z\}
\end{equation}
and so the output of our neuron is given by:
\begin{equation}\label{eqn: Neural network output}
	\hat{y} = \sigma\bigg(\sum_{i=0}^{n}w_i x_i\bigg)
\end{equation}
where $w_0 = b$ and $x_0 = - 1$ to account for the bias term.

The weights are initialized randomly when the network is instantiated and are updated during the training process.
In order to update the weights, a loss function must be specified.
This loss function will take the output of the neural network $\hat{y}$ and compare this prediction with the true value $y$ to assess the error.
Loss functions are too numerous to go through in detail here, so we refer interested readers to Goodfellow~\cite{goodfellow2016deep} for details.
For our purposes, we are interested in a classification task, and so we use the cross-entropy loss, defined as:
\begin{equation}
	J = -\sum_{c=1}^{M} y_{o,c} \ln(\hat{y}_{o,c})
\end{equation}
Where
$$y_{o,c} =
\begin{cases}
1 \quad\text{if class label $c$ is correct for observation $o$}\\
0 \quad\text{otherwise}	
\end{cases}
$$
and $\hat{y}_{o,c}$ is the predicted probability that observation $o$ belongs to class $c$.
Since our number of classes, $M$ equals 2, the cross-entropy loss simplifies to:
\begin{equation}
	J = -[y \ln(\hat{y}) + (1 - y) \ln(1 - \hat{y})]
\end{equation}

In order to update our weights, we must take the gradient of the loss function, $\nabla J$, and then update the weights of each layer by backpropogation: 
\begin{equation}
w_t = w_{t-1} - \alpha * \nabla J_{t-1}
\end{equation}
where $\alpha$ is the learning rate.
There are many optimization algorithms which can be used by neural networks, and this is an active area of research.
Throughout our text, we use stochastic gradient descent~\cite{hastie01statisticallearning} without momentum.

Much of the power of neural networks as compared to standard polynomial regression stems from their incredible ability to generalize to previously unseen data.
Cybenko~\cite{cybenko1989approximation} first proved that 2-layer neural networks using sigmoid activation functions can uniformly approximate any continuous function of n real variables with support in the unit hypercube.
This result has been extended several times to other activation functions and networks of bounded width and depth. 
Due to Lu \textit{et al.}~\cite{lu2017expressive}, we can state the following:

\begin{theorem}
Let $f: \mathbb{R}^n \to \mathbb{R}$ be a Lebesgue-measurable function satisfying
$$\int_{\mathbb{R}^n} \abs{f(x)} dx < \infty$$
then for any Lebesgue-integrable function $f$ and $\epsilon \in \mathbb{R}; \epsilon > 0$, there exists a fully-connected ReLU network $\mathcal{A}$ of width $d_m \leq 4 + n$ such that the function $F_{\mathcal{A}}$ represented by the neural network satisfies:
$$\int_{\mathbb{R}^n} \abs{f(x) - F_{\mathcal{A}}(x)}dx < \epsilon$$
\end{theorem}

Despite the strength of these results, representation learning and neural network interpretability are open questions.
At present, little is understood about the exact mechanism by which neural networks are able to learn, and what the meaning of the learned representation is.
Some theories exist and since many of them are not mutually exclusive, it stands to reason that several may be true.
In particular, we will consider the information bottleneck theory~\cite{tishby2015deep, fischer2020conditional} from a geometric point of view.

\section{Prior Work}
Our work leverages an expanded dataset from Watkins \textit{et al.}~\cite{watkins2018network} and one of our objectives, as in Watkins' work, is to build a model that sufficiently detects Android malware using the interarrival time of Internet Control Message Protocol (ICMP) ping packets.
In the literature, decision trees were used to classify traffic.
Other work on the dataset by Watkins' team more closely mirrors our own, and details are elaborated in Chapter~\ref{chap:three}.
In order to compare to Watkins' results to our own neural network results, we leverage a random forest from the Scikit-learn~\cite{scikit-learn} Python package to serve as a baseline.
Our work differs in that rather than just seeking to optimize our capability to detect malware, we use Watkins' dataset to evaluate the efficacy of different neural network architectures.

The use of Fourier transforms in neural networks has been of interest for some time, and there are several papers on the subject~\cite{osowski2002fourier, pratt2017fcnn, highlander2016very} which consider these applications.
The choice to explore Fourier transforms in convolutional neural networks is natural, as the dot product is much faster than a convolution operation that relies on a sliding kernel. 
We lean most heavily on the paper by Pratt \textit{et al.}~\cite{pratt2017fcnn} due to its recency and implementation details. 
Particularly, Pratt considers the impact of the convolution theorem within neural networks and uses the Fast Fourier Transform to quickly compute $\mathcal{F}(\kappa * u) = \mathcal{F}(\kappa) \odot \mathcal{F}(u)$, where $\mathcal{F}$ is the Fast Fourier Transform, $*$ denotes convolution, and $\odot$ denotes the Hadamard pointwise Product.
Ultimately, the paper shows that on the CIFAR-10 and MNIST datasets, the overall accuracy is lower than benchmark results - though the network trains and evaluates images much more quickly.
Interestingly, we found the opposite results, which we detail in Chapter~\ref{chap:three}.

Wavelet neural networks pioneered by Fujieda \textit{et al.}~\cite{fujieda2017wavelet} have shown promise for generalized convolution by abstracting them into downsampling and filtering in the spectral domain.
The results in the Fujieda paper were significant, as the network achieved better accuracy results than AlexNet on the target dataset while having approximately 1/4 the number of parameters.
In addition, the memory requirements and speed of the network were a significant improvement on the other architectures considered by Fujieda.
It is worth noting that implementation details from Fujieda are sparse, and so our implementation may differ from this reference implementation in some way, though the spirit and overall methods are the same.

Our work also considers and builds upon the Information Bottleneck theory of Neural Networks introduced by Tishby~\cite{tishby2015deep}.
The information bottleneck theory of deep learning suggests that the goal of supervised learning is to capture and efficiently represent the relevant information about the input data about the target data. 
In the process of creating a minimal sufficient statistic, a maximally compressed mapping of the input which minimizes mutual information is generated.
Tishby does this by demonstrating that the layered structure of the network creates a Markov chain of intermediate representations that forms the sufficient statistics.
The paper also suggests learning via information bottleneck, a technique we do not leverage.

Many of the issues with the information bottleneck theory are addressed in a controversial paper by Saxe~\cite{saxe2019information}.
Saxe argues compellingly that many of the issues with the saturation of nonlinearities such as tanh are not observed with ReLU.
Additionally, it argues - using empirical results - that networks which do not compress can still generalize. 
Saxe does not, however, argue that the fundamental conceit of the information bottleneck theory still holds and that an information theoretic approach to neural networks is still critical.

Fischer~\cite{fischer2020conditional} improves on Tishby's work by addressing Saxe's concerns and experimenting with both deterministic models as well as Variational Information Bottleneck~\cite{alemi2016deep} models. 
Fischer suggests that problems with robust generalization and lack of compression stem from models retaining too much information about the training data.
The Conditional Entropy Bottleneck model proposed by Fischer directly optimizes what he calls the Minimum Necessary Information criteria. 
Our work leans on the minimum necessary information criterion as a point of theory, and we detail it in Chapter~\ref{chap:two}.

\section{Content of this Thesis}
Neural networks largely remain a closed box and in practice, a lot of effort needs to be invested into feature engineering to achieve desired outcomes.
This is not necessarily true in the computer vision space where representations are learned from pixels, but is demonstrably true in domains such as information security, where terminology is not universally agreed upon and measurement can be difficult.
Additionally, feature engineering is crucial in systems where the relationship between measurements and events is not deterministic and easily confounded.

Wavelet transforms and the Fourier transform were considered for use as powerful tools from signal processing that have wide-ranging uses and implications well beyond our goal of malware classification.
Our initial expectation was that these tools would serve to enhance the ability of networks to learn by extracting features from the raw data, processing it into a form which would give the neural network a more robust feature set to learn from.
During the experiments, the transformations seemed not to alter the ability of the network to learn, and we looked to information theory for an explanation. 
In our consideration of the network as a manifold and the weights of the network as projections of our data as a coordinate system on that manifold which is optimized through gradient descent, information theory provided a framework on which to build understanding of our results. 
This work is presented in the following order, which front-loads this theory as a lens through which we can see our results.

\begin{itemize}
	\item In Chapter~\ref{chap:two}, we cover the necessary information theory to contextualize the results of this thesis. We cover common terminology, all of which is covered in greater depth in the canonical introduction to information theory by Cover and Thomas~\cite{coverthomas2006}. We move on to prove an important result about preservation of mutual information under homeomorphism. We introduce the information bottleneck and the concept of minimum necessary information,  which further contextualizes our results. We then discuss a geometric view of neural networks and how the optimal weights of the neural network can be viewed  as the orthonormal projection of the target onto the manifold of the neural network.
	\item In Chapter~\ref{chap:three}, we address our malware problem and dataset. We explain the four neural network architectures, and two baseline models we use in our experiments. We conduct transforms on our data and treat the raw, Fourier, and Wavelet-transformed data as well as a dataset of summary statistics in line with Watkins~\cite{watkins2018network}. We also run our malware dataset through Fourier and Wavelet neural networks inspired by prior work~\cite{pratt2017fcnn, fujieda2017wavelet} where the convolution theorem is exploited to learn representations. We consider the mutual information as before and address the accuracy results of the networks in question on our raw and summary datasets. We also capture the information plane representation of mutual information in the network similar to prior research~\cite{shwartz2017opening, saxe2019information}.
	\item In Chapter~\ref{chap:five}, we treat the MNIST dataset as our baseline for performance. This experiment was conducted to confirm our findings from our experiments on the malware dataset. We choose MNIST as our baseline dataset; due to the vast amount of literature on MNIST, it has been referred to as the "Drosophila of machine learning"~\cite{goodfellow2016deep}, making it suitable to contextualize our results. Our experimentation in this chapter is functionally the same as in Chapter~\ref{chap:three}, with minor modifications to the methodology where they would not apply to an image-based dataset, such as the elimination of the summary statistic dataset and the baseline detection models.
	\item Finally, Chapter~\ref{chap:conclusion} brings together the outcomes of our networks, baseline models, and information theoretic considerations of learning to explain our results. We consider avenues for further research and potential implications of our findings.
\end{itemize}

%
\chapter{Information Theory Preliminaries}
\label{chap:two}

\section{A Whirlwind Tour of Information Theory} 
Since our work makes use of information theory, it is helpful to cover the core terminology.
For the following proofs, all random variables are assumed to be discrete. 
This is both because binary computers have only finite precision, which means they are not ``truly'' continuous, and discrete information theory is a more mature science in that many foundational results are proven only in the discrete case.

First, we define entropy as the measure of uncertainty of a single random variable.
Let $X$ be a random variable with alphabet $\mathcal{X}$ and probability mass function $p(x) = \\Pr\{X = x\}, x \in \mathcal{X}$.
Then the entropy of $X$, represented $H(X)$ is defined as
\begin{equation}\label{eqn: Entropy}
H(X) = -\sum_{x \in \mathcal{X}} p(x) \log{p(x)}	
\end{equation}
Here our logarithm is to the base 2, as information is most commonly represented as bits.
We maintain this definition of the logarithm throughout. 

In the case where we are examining two random variables, for example, a dataset and its labels, we may want to consider the joint and conditional entropy of those random variables.
The joint density of a pair of discrete random variables $(X, Y)$ with joint distribution $p(X, Y)$ is:
\begin{equation}
H(X, Y) = - \sum_{x \in \mathcal{X}} \sum_{y \in \mathcal{Y}} p(x, y) \log{p(x, y)}	
\end{equation}

\noindent and the differential entropy $H(Y | X)$ as:
\begin{equation}\label{eqn: Differential Entropy}
H(Y|X) = -\sum_{x \in \mathcal{X}} \sum_{y \in \mathcal{Y}} p(x, y) \log{p(y|x)}	
\end{equation}

Therefore, we define mutual information, $I(X; Y)$ as the relative entropy between the joint distribution and the product distribution:
\begin{equation}\label{eqn: Mutual Info}
	I(X; Y) = \sum_{x, y} p(x, y) \log{\frac{p(x,y)}{p(x) p(y)}} = H(Y) - H(Y|X)
\end{equation}
Mutual information is an important quantity for us, since it is a measure of dependence between two random variables. 
Specifically, it provides a measure of information obtained about one random variable by the observation of another random variable.
Therefore, if we can observe $X$ and want to predict $Y$, we would like for the mutual information between X and Y to be high. 

We define the Kullback-Leibler divergence $D(p||q)$ between two probability mass functions $p(x)$ and $q(x)$ to be:
\begin{equation}\label{eqn: KL Divergence}
D(p||q) = \sum_{x \in \mathcal{X}} p(x) \log{\frac{p(x)}{q(x)}}	
\end{equation}

Lastly, we introduce the data processing inequality.
Assume that three random variables $X, Y,$ and $Z$  form a Markov chain denoted $X \to Y \to Z$.
Let $Z$ depend only on $Y$ and let $Z$ be conditionally independent of $X$. 
Then the data processing inequality shows that no local manipulation of the data can improve inferences drawn from that data.
By the chain rule, we can expand mutual information as follows:
\begin{align}
I(X; Y, Z) & = I(X; Z) + I(X; Y|Z) \\
& = I(X; Y) + I(X; Z|Y)
\end{align}
Since $X$ and $Z$ are conditionally independent given $Y$, we have $I(X; Z|Y) = 0$. 
Since $I(X; Y|Z) \geq 0$, we have:
\begin{equation} \label{eqn:DPI}
I(X; Y) \geq I(X; Z)	
\end{equation}
which is known as the data processing inequality. 

Detailed derivation of all above results are available in \cite{coverthomas2006}.

\section{Invariance of Mutual Information}\label{invariance}

We begin by proving the invariance of mutual information under homeomorphism, based on a similar proof~\cite{kraskov2004estimating}.

\begin{theorem}[Invariance of Mutual Information under Homeomorphism]\label{thm:MI invariance}
Take two random variables $X$ and $Y$ where $Y$ is the set of labels of $X$.
Let $X' = \psi(X)$, where $\psi$ is a smooth and uniquely invertible map (a homeomorphism).
Then since $X$ is a random variable, $X'$ is a random variable as long as $\psi$ is well-defined for the range of $X$.
Thus, $I(X'; Y) = I(X; Y)$.
\end{theorem}
\begin{proof}
Given the Jacobi determinant $J_X = ||\partial X/ \partial X'|| = ||\partial X / \partial \psi(X)||$, we observe that the joint distribution of $X'$ and $Y$ is given by: $f_{X', Y}(x', y) = J_X(x')f(x, y)$
\begin{align}
I(X'; Y) & = \int \int dx' dy f(x', y) \log \frac{f(x',y)}{f_{x'}(x')f_{y}(y')} \\
& = \int \int dx dy f(x, y) \log \frac{f(x, y)}{f_{x}(x)f_{y}(y)}\\
& = I(X; Y)
\end{align}
\end{proof}

\section{Minimum Necessary Information and Information Bottleneck}\label{MNI_IB}
Naftali Tishby and Noga Zaslavsky introduced the information bottleneck theory of neural networks~\cite{tishby2015deep} as a way of explaining the theoretical generalization bounds of neural networks.
In particular, Tishby and Zaslavsky show that any deep neural network can be quantified by the mutual information between the input, hidden layers, and the output variable by way of information per the data processing inequality, Equation~\ref{eqn:DPI}.
Neural networks satisfy the information bottleneck optimality equation:
\begin{equation}
\min_{p(z|x):Y \to X \to Z} |I(Z;X) - \beta I(Z; Y)| , \quad\beta > 0	
\end{equation}
Where $Y$ are the true labels, $X$ is the observed data about $Y$, and $Z$ is the learned representation. 
The information bottleneck learns the representation $Z$ subject to the above constraint, where $\beta$ controls the strength of the constraint.
The standard cross-entropy loss is recovered as $\beta \to \infty$. 
We do not concern ourselves with the existence of the compression phase addressed by Saxe~\cite{saxe2019information} but instead observe that the information bottleneck optimality equation holds irrespective of whether fitting and compression happen in sequence or simultaneously.
Additionally, the value of the information bottleneck to this work is in its implication that a neural network seeks to learn a representation, $Z$, which retains a maximal amount of information about $Y$ and a minmal amount of information about $X$.
Further work by Alemi \textit{et al.}~\cite{alemi2016deep} suggests refinements on the information bottleneck theory that we do not discuss in detail here. 

The Minimum Necessary Information as defined by Fischer~\cite{fischer2020conditional} consists of three components for a learned representation:
\begin{enumerate}
	\item \textbf{Information} We would like a representation $Z$ that captures useful information about a dataset $(X, Y)$.
	\item \textbf{Necessity} The value of information to accomplish a task. In this case, predicting $Y$ given $X$ using our representation $Z$. That is, $I(X; Y) \leq I(Y; Z)$
	\item \textbf{Minimality} Given all representations that can solve the task, we prefer the one that retains the smallest amount of mutual information. That is, $I(X; Y) \geq I(X; Z)$.
\end{enumerate}
As mentioned in our discussion of Equation~\ref{eqn: Mutual Info}, the higher the mutual information between this representation $Z$ and our desired prediction $Y$, the better our predictions will be. 

Using Fischer's definitions of necessity and minimality, we see that there is a point called the ``MNI Point'':
\begin{equation}
	I(X; Y) = I(X; Z) = I(Y; Z)
\end{equation}
This equation may not be satisfiable, since for any representation $Z$ given a dataset $(X, Y)$, there is a maximum value we are subject to:
\begin{equation}\label{eqn:MNI}
	1 \geq D(X||Z) = \sup_{Z \leftarrow X \rightarrow Y}\frac{I(Y; Z)}{I(X; Z)}
\end{equation}
Where $D(X||Z)$ is the KL divergence given in Equation~\ref{eqn: KL Divergence} and we achieve equality if and only if the Markov chain $X \to Y$ is deterministic.


\section{Information Geometry of Neural Networks}
A neural network, as mentioned briefly in Section \ref{intro:nn}, is a form of connectionist machine learning that is a universal approximator under minor assumptions about the activation function~\cite{goodfellow2016deep}.
In particular, a neural network connects many artificial neurons each of which receive input $x$ and emit an output that is a prediction of $y$. 
From Equation~\ref{eqn: Neural network output}, the forward pass of a single neuron gives us:
\begin{equation}
\hat{y} = \sigma\bigg(\sum_{i=1}^{n}w_i x_i + \beta_{i} \bigg) = \sigma(w x + \beta)	
\end{equation}
Where $\sigma$ is an activation function meeting the aforementioned assumptions, and $\beta$ is a bias vector.

Let $\mathcal{S}$ be the manifold of neural network outputs $\mathcal{S} = \{\sigma(w x + \beta) : w \in \mathbb{R}^n, \beta \in \mathbb{R}\}$ parametrized by $w$ and $\beta$. 
We picture the manifold $\mathcal{S}$ as an (n+1)-dimensional smooth surface in the infinite-dimensional space of functions on $\mathbb{R}^n$. 
Assume that our data is generated by some function $g$ such that $y = g(x)$. 
Then if $g \in \mathcal{S}$ there exist $w^* \in \mathbb{R}^n$, $\beta^* \in \mathbb{R}$ such that we have an exact representation of $g$.
In general, most target functions are not in $\mathcal{S}$ and so we must train the values
\begin{equation}
	(w^*, \beta^*) = \arg \min_{w, b} dist(g, \mathcal{S})
\end{equation}
which correspond to the coordinates of the orthogonal projection of $g$ onto the surface $\mathcal{S}$.
Thus, our optimal parameter $\xi^* = (w^*, \beta^*)$, if it exists, is given by:
\begin{align} \label{eqn:optimal_xi}
	\xi^* &= \arg \min_{w, \beta} dist(p(x, z), p(x, y; w, \beta)) \\ 
	&= \arg \min_{\xi} D(p(x, z) || p(x, y; \xi))
\end{align}
\noindent where D is the KL divergence specified in Equation~\ref{eqn: KL Divergence}.
The relationship between the joint probability distribution and mutual information is specified in Equation~\ref{eqn: Mutual Info}. 
Since we optimize $\xi$ with respect to $D(\hat{y}||{y})$, as we approach the MNI point, the KL divergence approaches zero.

Returning to our derivation in Theorem~\ref{thm:MI invariance}, we see that the optimal parameter $\xi^*$ and the MNI point given by Equation~\ref{eqn:MNI} can be achieved for $X' = \phi(X)$, where $\phi$ is a homeomorphism, since any projection onto $\mathcal{S}$ will still be on the manifold, translated by the map $\phi$.
Thus, we conjecture that the ability of a network to learn a representation that is predictive of $y$ is invariant to homeomorphism on the input data manifold.
Returning to Equation~\ref{eqn:MNI}, when $Z$ is replaced by $\hat{y}$, we find that as we approach the global minimum of the loss surface, we are minimizing $D(\hat{y}||y)$, which allows us to show:
\begin{align} \label{eqn:2.19}
\min D(\hat{y}||y) & = \min \sup_{\hat{Y} \leftarrow X \rightarrow Y}\frac{I(Y; \hat{Y})}{I(X; \hat{Y})} \\
& = \min \sup_{f(X, \xi) \leftarrow X \rightarrow Y} \frac{I(Y; f(X, \xi))}{I(X; f(X, \xi))} \\
& = \min \sup_{f_X(\xi) \leftarrow X \rightarrow Y} \frac{I(Y; f_X(\xi))}{I(X; f_X(\xi))}
\end{align}
where $f_X(\xi)$ is the neural network with input $X$ and given parameters $\xi$.
So our learning process is minimizing the mutual information between $Y$ and $\xi$, while maximizing the mutual information between $X$ and $\xi$.
This inequality also holds for $\xi'$, the optimal set of parameters for the input $X'$.
This theory is discussed in the context of our experiments in Chapter \ref{chap:conclusion}.
%
\chapter{Experiment 1 - Watkins Malware Dataset}
\label{chap:three}
Neural networks have demonstrated some success in the security domain~\cite{raff2018malware} and so we have applied them to the Watkins~\cite{watkins2018network} dataset.
This dataset consists of interarrival times for packets sent to Android devices, some of which were running malware.
Detecting malware via network traffic is an important problem, and this interpretation of the problem is crucial for addressing the situation where a network owner cannot install an antivirus agent on a device that may be infected with malware.
This is a common situation when personal devices are introduced to a corporate network; being able to detect malicious software on a device without having an agent on the device provides a tremendous benefit to network defenders.

\section{Methodology}
This experiment proceeds in three parts.
The first considers our standard models: fully-connected neural network, convolutional neural network, support vector machine, and random forest.
These models process our malware datasets in alignment with the methods elaborated in Section~\ref{data}.
In experiment 1.a, the neural networks are trained for a maximum of 30 epochs and we use the early stopping technique to prevent overfitting.
Early stopping will end training early when some condition is met - in our case, we stop early if the network's loss has not decreased by 0.001 or more for two consecutive training epochs.

In experiment 1.b, we consider only the raw data across the standard models as well as the Fourier neural network and wavelet neural network.
The Fourier neural network and wavelet neural network differ from a conventional convolutional neural network by performing an in-layer transformation before the activation function is applied, exploiting the convolution theorem.
As in our previous experiment, in order to ensure that deviations in the dataset did not induce significant variation in accuracy, 100 trials were run, and the accuracy and mean step time for all trials was averaged.

In experiment 1.c, the network is trained for 1000 epochs without early stopping, and at each training epoch, the mutual information between the labels and the network, $I(Y; M)$ and the mutual information between the data and the network, $I(X; M)$ is computed.
This differs from the previous experiment as we do not concern ourselves with accuracy but instead wish to see and plot the change in mutual information during training.
By running for 1000 epochs, we significantly overfit the training set, making this a poor approach for optimizing accuracy.
The details of this plot are described below in Section~\ref{infoplane}

\subsection{Mutual Information Computation}\label{MI computation}
In each experiment, the following data are collected for each epoch:
\begin{enumerate}
	\item The L2 norm of the weights
	\item The mean of the gradients
	\item The standard deviation of the gradients
	\item The post-activation output of each layer for the test set. 
\end{enumerate}
These data are then stored in a file.
After training, the data is loaded from the files.
The entropy of the activity is computed by considering the KL-based upper bound on the entropy using techniques from Section 4 of Kolchinsky and Tracey~\cite{kolchinsky2017estimating} to yield the entropy of the layer $H(M)$.
This estimate is:
\begin{equation}
	H(M) = -\sum_{i} p_i \ln \sum_{j} p_j \exp(-D(m_i || m_j))
\end{equation}
where $p$ is either the probability density of the dataset, $X$, or the probability of the label $Y$, and $m$ is the probability density estimate of our network layer, $M$.
For the entropy with respect to the labels, individual label probabilities are computed and used with the entropy of the activity to compute the conditional entropy of the activity given the label probabilities, giving us $H(M|Y)$.
This is used in conjunction with our computation of $H(M)$ so that we can compute the mutual information $I(Y; M) = H(M) - H(M|Y)$.
 
These two mutual information values are then used to display information plane data as plotted in Figure~\ref{fig:malware infoplane fc}, Figure~\ref{fig:mnist fc infoplane} and others.
These calculations are identical to the methods used in Saxe~\cite{saxe2019information}.

\section{Data}\label{data}
We leveraged four different datasets: Raw, Fourier-transformed, Wavelet-transformed, and a dataset consisting of summary statistics.
The summary statistics of the first three are captured in Table~\ref{Tab:summary} - we did not compute summary statistics for the dataset of summary statistics.  
Our data consisted of 98 legitimate applications and 120 pieces of malware, which were collected by Yu and Li~\cite{yu2018network}.
This gives us a dataset that is approximately 55\% malware and 45\% benignware.
While this distribution is not reflective of real environments where malware is significantly rarer than benign applications, we do not adjust for this disparity since our tolerance for alerting on benign applications is much higher than our tolerance for not detecting malicious applications. 
For each application, five trials were conducted where the interarrival time was collected for each of 100 ICMP ping packets, yielding a total dataset of 1090 trials.
Further details of the data collection can be found in~\cite{yu2018network, watkins2018network}.
One interesting effect of performing the transforms on the dataset is that while the continuous wavelet transform reduces our variance significantly and slightly normalizes the dataset, the Fourier transform has the opposite effect, introducing tremendous amounts of noise into the dataset.

\subsection{Raw data}
This is the data as described above, captured by Yu and Li in accordance with Watkins~\cite{watkins2018network}.
In this dataset, only the raw measurements are used in a 100-dimensional row vector, with a label of 0 for benign and 1 for malicious.

\subsection{Fourier data}
The Fourier data is a copy of the raw data under the Fourier transform.
In particular, since our raw data is given by a single 100-dimensional row vector, it is a direct mapping of that row vector under the Fast Fourier Transform as provided by the numpy library.

\subsection{Wavelet data}
The wavelet dataset is a copy of the raw data under a continuous wavelet transform. 
The Morlet wavelet is used for the transform for several reasons:
First, it is a wavelet that allows us to maintain the dimensionality of our data, making it easier to compare in performance and to re-use neural network architectures.
Secondly, the Morlet wavelet is closely related to human perception~\cite{mallat1999wavelet, daugman1985uncertainty}, providing a small connection to the human brain conception of neural networks.
Lastly, the Morlet wavelet is uniquely invertible, which is not the case for all potential mother wavelets. 

\subsection{Summary data}
The summary data leverages the seven features used by Yu and Li: arithmetic mean, standard deviation, variance, maximum, minimum, geometric mean, and harmonic mean. 
These features were fed to the classifier based on the sample they were captured from as a 7-dimensional row vector.

\renewcommand{\thefootnote}{*} 
\begin{table}[h]
\centering
\begin{tabular}{l|llll}
\textbf{Dataset Name} & \textbf{Mean} & \textbf{Median} & \textbf{Mean Var.} & \textbf{Median Var.} \\\cline{1-5}
Raw         & 27.43    & 10.07    & 8329.96    & 7663.89 \\
Fourier       & 58.60\footnotemark    & -1.45    & 12229005.34    & 7418186.12 \\
Wavelet        & 1.30    & -.072    & 1887.5    & 1507.16 \\             
\end{tabular}
\caption{Dataset Summary Statistics}
\label{Tab:summary}
\end{table}
\footnotetext{There is an extremely small, but non-zero imaginary part, on the order of $10^{-19}i$}

\renewcommand{\thefootnote}{1}
The small but non-zero imaginary part in the Fourier data required implementation of methods from Trabelsi \textit{et al.}~\cite{trabelsi2017deep} to achieve acceptable results.

\section{Models}\label{model_descriptions}
In our experiments, we leveraged the following models:
\begin{itemize}
\item Fully Connected Neural Network
\item Convolutional Neural Network
\item Fourier Neural Network
\item Wavelet Neural Network
\item Random Forest
\item Support Vector Classifier	
\end{itemize}

The summary statistic dataset was not used with the convolutional neural network, nor was it used with the Fourier or wavelet neural networks because there is no spatial relationship between the data and so convolution offers no benefit.
Additionally, only the raw data was processed by the Fourier and Wavelet neural networks.
Though these networks are capable of processing the transformed data, there is no obvious benefit to transforming already-transformed data in-network.

All code\footnote{Code is available at the following url: \url{https://github.com/erickgalinkin/jhu_masters}} was written in Python, using the Tensorflow 2, PyTorch, and Scikit-learn libraries.
Only the baseline models - the random forest and support vector machine - described in \ref{other models} used the Scikit-learn library, and only the Wavelet Convolutional network described in \ref{wavelet cnn} used PyTorch.
The remaining models all used the Tensorflow framework.
In the case of our non-standard neural networks, we consider the work of Pratt~\cite{pratt2017fcnn} and Fujieda~\cite{fujieda2017wavelet}.
Both the Fourier and Wavelet neural networks take advantage of the convolution theorem - that is, given two functions $f$ and $g$,
\begin{align*}
(f * g)(t) & = \int_{-\infty}^{\infty} f(\tau)g(t-\tau)dt \\
& = \int_{\mathbb{R}^n}f(x) e^{-2\pi i x \nu} dx	 \cdot \int_{\mathbb{R}^n}g(x) e^{-2\pi i x \nu} dx\\	
& = \mathcal{F}\{f\}(\nu) \cdot \mathcal{F}\{g\}(\nu)
\end{align*}

and in the inverse, we get:
$$f \cdot g = \mathcal{F}^{-1}\{\mathcal{F}\{f\} * \mathcal{F}\{g\}\}$$

This allows us to avoid the high computational cost of performing a convolution via the sliding-tile method and instead potentially take advantage of the convolution theorem to perform convolution at the speed of the dot product.
We further elaborate on the architecture below.

For our neural networks, we use the Tensorflow standard stochastic gradient descent optimizer, with a learning rate of 0.001.
Results with other optimizers have been promising, and Adam~\cite{kingma2014adam} has been the optimizer of choice for many deep learning applications in the past few years, though we do not leverage it here. 
Hardware specifications on which these experiments ran is in Appendix~\ref{append:one}.

\subsection{Fully-Connected Neural Network}
The fully-connected neural network architecture is a basic multi-layer perceptron that accepts a 100-dimensional row vector.
This vector is then fed to three densely connected hidden layers, each with 256 ReLU-activated neurons.
The fourth and final output neuron is a single sigmoid-activated layer, which provides a probability of maliciousness.

\subsection{Standard Convolutional Neural Network}
Our convolutional neural network is a sequential model that accepts a 100-dimensional row vector as input.
This input is processed by two convolutional layers, both with 256 neurons.
The first has a kernel size of 5 and a stride size of 1, and the second has a kernel size of 3 with a stride of size 1.
The output of the second convolutional layer is processed by two densely connected layers of 128 neurons each. 
The final layer consists of a sigmoid-activated output layer, the same as the fully-connected neural network.

\subsection{Fourier Convolutional Neural Network}
Our Fourier ``Convolutional'' neural network is identical architecturally to our standard convolutional neural network, only with the convolutional layers replaced by Fourier layers.
Here, we put the word convolutional in quotes due to the fact that no actual convolution is performed.
To be more intellectually honest, we should refer to this network instead as a ``Fourier Transform Cross Product Network'', though this may confuse readers unfamiliar with the relationship.
In the interest of broad understanding, the term convolutional neural network is used when it helps clarify meaning even in spite of being a slight misnomer.

Specifically, the Fourier Convolutional Neural Network leverages a custom Fourier layer that moves the data into Fourier space via the Fast Fourier Transform and then multiplies the transpose of the weight matrix with the input to the matrix.
Specifically, given an input $X^{(n)}$ where the superscript is not an exponent, but instead indicates the layer of the input, the Fourier layer $\ell$ acts on $X$ to give an input to our next layer:
\begin{equation}
X^{(n+1)} = \ell^{(n)}(X^{(n)}) = \sigma(\mathcal{F}^{-1}(\mathcal{F}(X^{(n)})\cdot \mathbf{W}^{(n)\top}))
\end{equation}
\noindent Where $\mathcal{F}$ is the Fast Fourier Transform, $\sigma$ is the activation function - ReLU in this case - and $\mathbf{W}$ is the weight matrix for layer n.

\subsection{Wavelet Convolutional Neural Network} \label{wavelet cnn}
The Wavelet Convolutional Neural Network implements similar functionality to our Fourier Neural Network, using the Discrete Wavelet Transform in lieu of the Fourier transform.
Due to the fact that there is a time component and a frequency component, the wavelet neural network has a different in-layer dimensionality than our other models but is otherwise identical.

In our Wavelet Convolutional Neural Network, we take a 100-dimensional row vector as input.
This input is then sent to the ``wavelet layer'' where it undergoes a Daubechies discrete wavelet transform.
There are a very large number of wavelets which can be used in the discrete wavelet transform~\cite{mallat1999wavelet}, but the Daubechies wavelet is easy to put into practice and has a unique inverse everywhere, so we use it here.
The output is cast to a tensor that is multiplied against the transpose of the weight tensor.
This output then undergoes an inverse discrete wavelet transform with respect to the same mother wavelet.

\subsection{Baseline Models} \label{other models}
Two baseline models were considered on these datasets.
The first is the random forest model provided in the Scikit-learn library with no hyperparameter tuning.
Decision tree models are generally good at classification tasks~\cite{hastie01statisticallearning} but are weak classifiers that are sensitive to variance.
Random forests are the result of averaging a large collection of de-correlated trees and provide a good benchmark as a na\"ive model - in the respect that it is untuned - for classification.
Random forests are also performant in the respect that they train and evaluate examples quickly, relative to neural networks.
This makes them common for use in industry.

The other benchmark model is a Support Vector Classifier, again provided by the Scikit-learn library.
The rationale for using a Support Vector Machine is that we wanted to see if some hyperplane could be learned that would separate the data.
This model was again, na\"ive in the respect that it was merely the ``out of the box'' model, and so the classifier was built on top of the radial basis function kernel.
Details of the Support Vector Classifier can be found in James~\cite{james14introduction} or the Scikit-learn documentation~\cite{scikit-learn}.

\section{Results} \label{malware results}
We split the results here into three subsections for clarity, first presenting and discussing the malware data transformations with respect to the algorithms they were tested on. 
We then turn to the ways that all six architectures performed on the raw data.
Finally, we discuss the information plane of our neural networks.

\subsection{Malware Dataset Transformation}
Our results for the transformed datasets, contained in Table~\ref{Tab:malware_test} show our test accuracy and the mean time per batch for each neural network.
The time per batch is not available for the baseline models.

\begin{table}[h]
\centering	
\begin{tabular}{l|ll}
\textbf{Data and Architecture Combination} & \textbf{Test Accuracy} & \textbf{Mean Step Time} ($\mu$s) \\\cline{1-3}
Raw, Fully-Connected NN            & 63.40\%         & 13\\
Summary, Fully-Connected NN        & 55.14\%         & 13\\
Fourier, Fully-Connected NN        & 59.88\%         & 12\\
Wavelet, Fully-Connected NN        & 61.95\%         & 12\\
\hline
Raw, Convolutional NN              & 72.89\%         & 54\\
Fourier, Convolutional NN          & 70.81\%         & 56\\
Wavelet, Convolutional NN          & 70.77\%         & 52\\
\hline
Raw, Random Forest                 & 80.28\%         & N/A\\ 
Summary, Random Forest             & 76.91\%         & N/A\\
Fourier, Random Forest             & 79.63\%         & N/A\\
Wavelet, Random Forest             & 79.80\%         & N/A\\
\hline
Raw, Support Vector Classifier     & 65.77\%         & N/A\\    
Summary, Support Vector Classifier & 55.28\%         & N/A\\  
Fourier, Support Vector Classifier & 55.28\%         & N/A\\  
Wavelet, Support Vector Classifier & 55.28\%         & N/A           
\end{tabular}
\caption{Classifier accuracy on transformed datasets}
\label{Tab:malware_test}
\end{table}

In terms of accuracy, we find that the random forest on the raw data performs best, followed closely by the random forest on the wavelet-transformed data, and third the random forest trained on the Fourier-transformed data. 
On all datasets, the random forest classifier outperforms all other classifiers on that same dataset.
Notably, when we compare accuracy by model, we find that for the fully-connected neural network, our maximum average accuracy is 63.40\%, while our minimum average accuracy is given by the summary statistics.
Excluding the summary statistic data, the difference between the highest average accuracy and lowest average accuracy for fully-connected neural networks is 3.52\%, a very small margin. 
Comparatively, for the convolutional neural network, our delta is 2.12\%, again - quite small. 
Similarly, the random forest performs near the 80\% mark universally, irrespective of representation, and performs worst on the summary statistic dataset.

\subsection{In-network Data Transformation}
For our in-network data transformations, we consider only the raw dataset. 
The test accuracy and mean time per batch for both the Fourier and Wavelet neural network are contained in Table~\ref{Tab:test_arch} along with the results for the other four architectures on which the raw data was tested.

\begin{table}[h]
\begin{tabular}{l|ll}
\textbf{Architecture}  & \textbf{Test Accuracy} & \textbf{Mean Step Time} ($\mu$s) \\\cline{1-3}
Fully-Connected NN            & 63.40\%         & 13\\
Convolutional NN              & 72.89\%         & 54\\  
Fourier NN                    & 63.27\%         & 143\\
Wavelet NN                    & 74.85\%         & 228\\
Random Forest                 & 80.28\%         & N/A\\ 
Support Vector Classifier     & 65.77\%         & N/A       
\end{tabular}
\caption{All classifier accuracy on raw dataset only}
\label{Tab:test_arch}	
\centering
\end{table}

It is worth noting that one of the primary motivations for replacing the sliding-tile convolution method with a Fourier or Wavelet method is the performance gains identified by others~\cite{pratt2017fcnn}.
However, as we show, the Fourier and Wavelet networks are significantly slower than their untransformed counterparts on this dataset.
We conclude that the computational overhead of performing a transform and its corresponding inverse transform outweighs the speed-up gained by eliminating the sliding-tile convolution on smaller datasets, and the method as demonstrated in Pratt~\cite{pratt2017fcnn} should be reserved for relatively large images, where convolution is already slow.
In our case, we see a 2.65x increase in step time between a standard convolution and the Fourier method. 
Unfortunately, our activation functions do not behave nicely in the Fourier or Wavelet domain, as these functions operate linearly with respect to the space and so an inverse transform must be applied. 
The question of using a novel convolution operator and conducting the activation in that space has been addressed by Chakraborty~\cite{chakraborty2019surreal} but goes well beyond the question of simply adapting an activation function to the Fourier or Wavelet space.
The search for a spectral activation function remains an open question.

\subsection{Malware Dataset Information Plane Analysis}\label{infoplane}
Figure~\ref{fig:infoplane example} displays a zoomed-in view of the information plane for our malware dataset and neural network.
On the x axis is the mutual information $I(X;M)$, computed as described in Section~\ref{MI computation}.
On the y axis is the mutual information $I(Y;M)$. 
Optimally, we want to see high values on the y axis and lower values on the x axis for each layer - this would suggest that the learned representation in the neural network, $M$, requires relatively little data about $X$ to reliably predict $Y$.
In this figure, each layer is plotted independently. 

\begin{figure}[h!]
\begin{center}
\includegraphics[width=\textwidth]{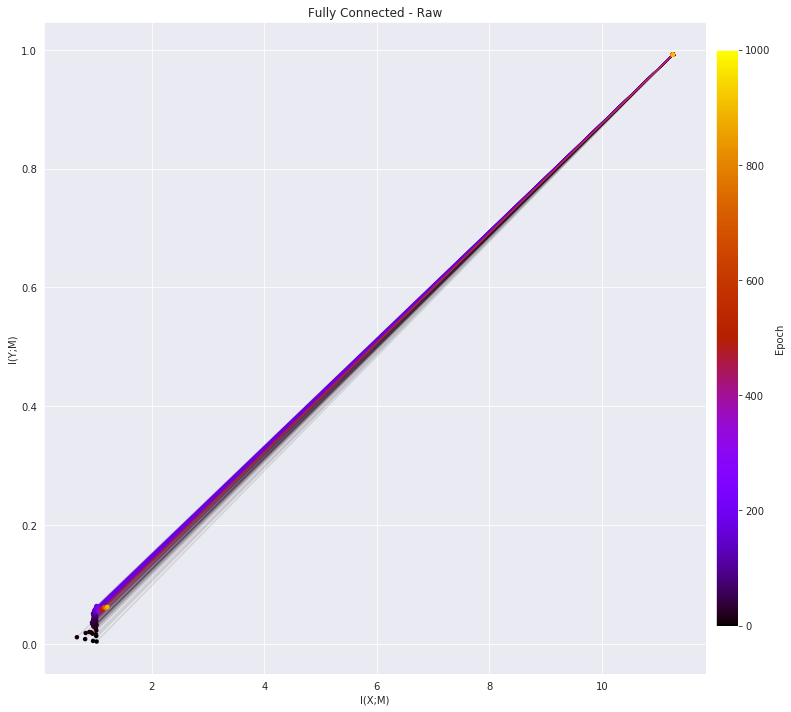}
\caption{Large plot for Fully-Connected Neural Network Information Plane on Raw Data. Produced using the upper bound and binning methodology from Saxe~\cite{saxe2019information} over 1000 epochs.}
\label{fig:infoplane example}
\centering
\end{center}
\end{figure}

The cluster of data points on the lower-left hand side represent the output layer, which gains slightly better predictive ability about the data throughout the 1000 epochs. 
The shift toward the right in later epochs is suggestive of overfitting the dataset, and $M$ containing more information about $X$.
Meanwhile, the shift upward, particularly early on, indicates the network improving the amount of mutual information between $M$ and $Y$.
We define the information plane as in Tishby~\cite{tishby2015deep}: the plane of the mutual information values that each layer preserves on the input and output.
In the upper right of the plot, we see what appears to be a single point - this is all three hidden layers of the neural network, which do not see any change in mutual information.
We verified during training that the weights were changing as expected in all hidden layers, and the network loss went down throughout training; only the mutual information did not change.
The low level of mutual information may be due to the weak correlation relationship between $X$ and $Y$, which has a bivariate correlation of 0.2629 for the raw data.

We can see in Figure~\ref{fig:malware infoplane fc}, subplots A, B, and C, that the amount of information changes very little. There is a high level of mutual information about $Y$ and $X$ captured in the hidden layers, while the output layer has almost no information about $Y$ and only learns less than 2 bits of information about $X$.
Given how similar the accuracies for the fully-connected neural network were - as can be seen in Table~\ref{Tab:malware_test} and how similar subplots A, B, and C are in Figure~\ref{fig:malware infoplane fc}, it's clear that the learned representations capture the same amount of information about the target labels.
With respect to our summary dataset information plane in subplot D, we note that the graph looks more like a scatter plot than a line chart, seemingly because of the nature of the transformation - that is, the entire representation of the data is changed. 

\begin{figure}[h]
\begin{center}
\includegraphics[width=\textwidth]{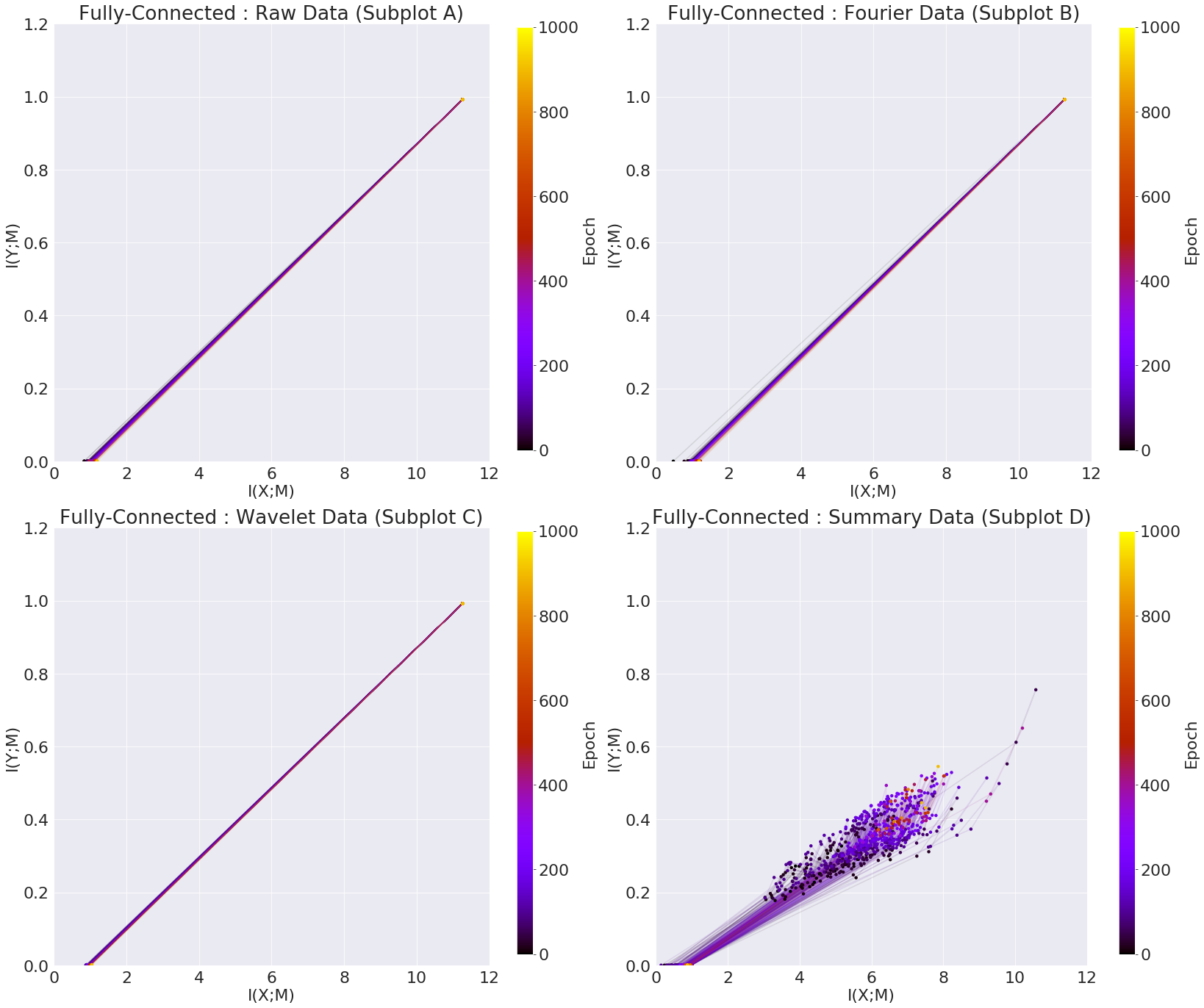}
\caption{Fully-Connected Neural Network Information Plane for four malware data sets}
\label{fig:malware infoplane fc}
\centering
\end{center}
\end{figure}

As we note in our discussion of minimum necessary information in Equation~\ref{eqn:MNI}, we achieve optimality only when $X$ uniquely determines $Y$, which does not appear to be the case for our dataset. 
It is worth noting that all of the information planes in Figure~\ref{fig:malware infoplane fc} aside from the summary data in subplot D do not change their mutual information for the hidden layers and converge to the same mutual information for the output layer - the only layer which sees a change in mutual information.
It was experimentally verified that although the weights and entropies of each individual layer did change throughout training, the mutual information for the hidden layers remained stationary across 1000 epochs on the non-summary datasets.

\begin{figure}[h!]
\begin{center}
\includegraphics[width=\textwidth]{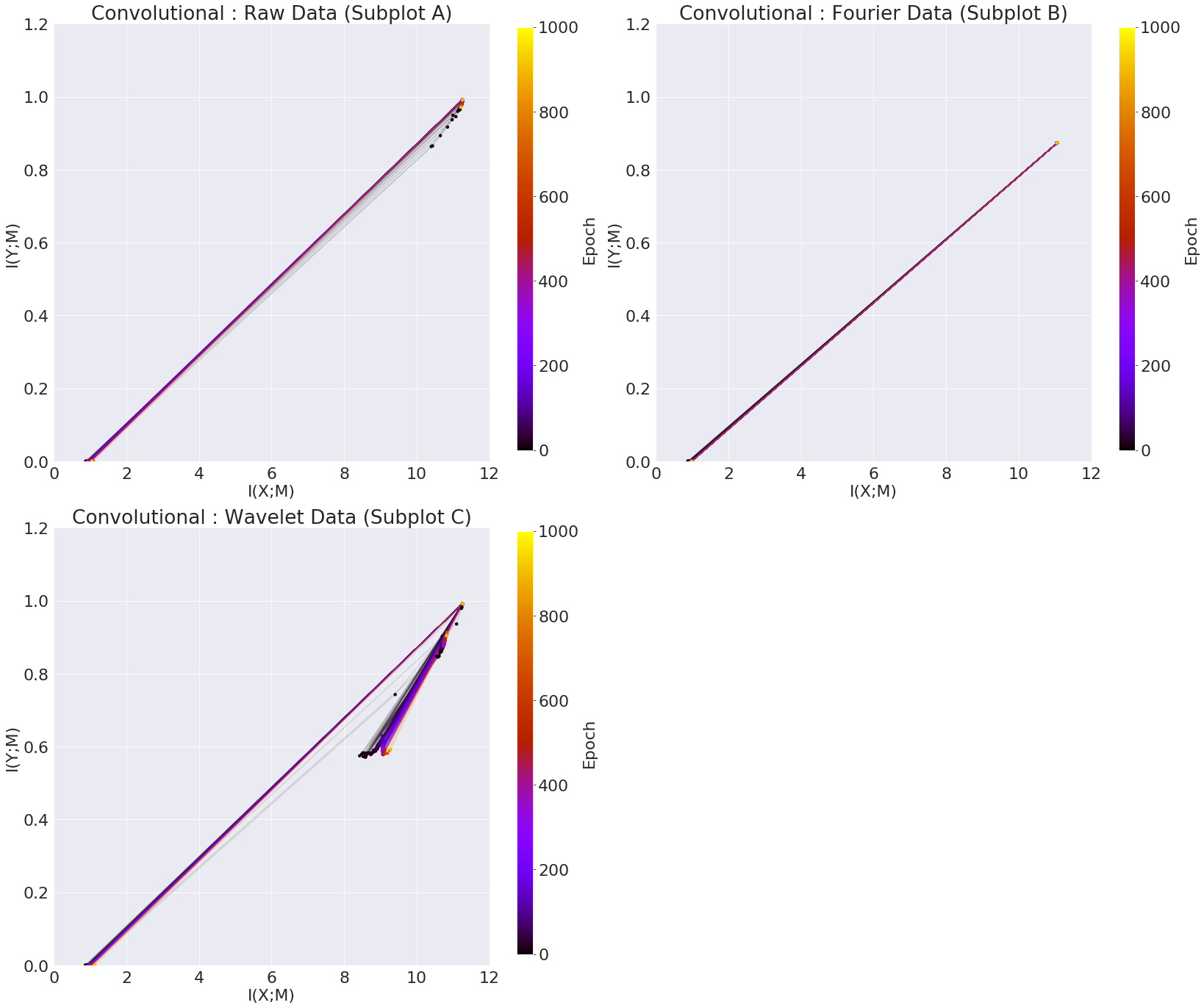}
\caption{Information Plane for Convolutional neural network for three data sets}
\label{fig:malware infoplane conv}
\centering
\end{center}
\end{figure}

\begin{figure}[h!]
\begin{center}
\includegraphics[width=\textwidth]{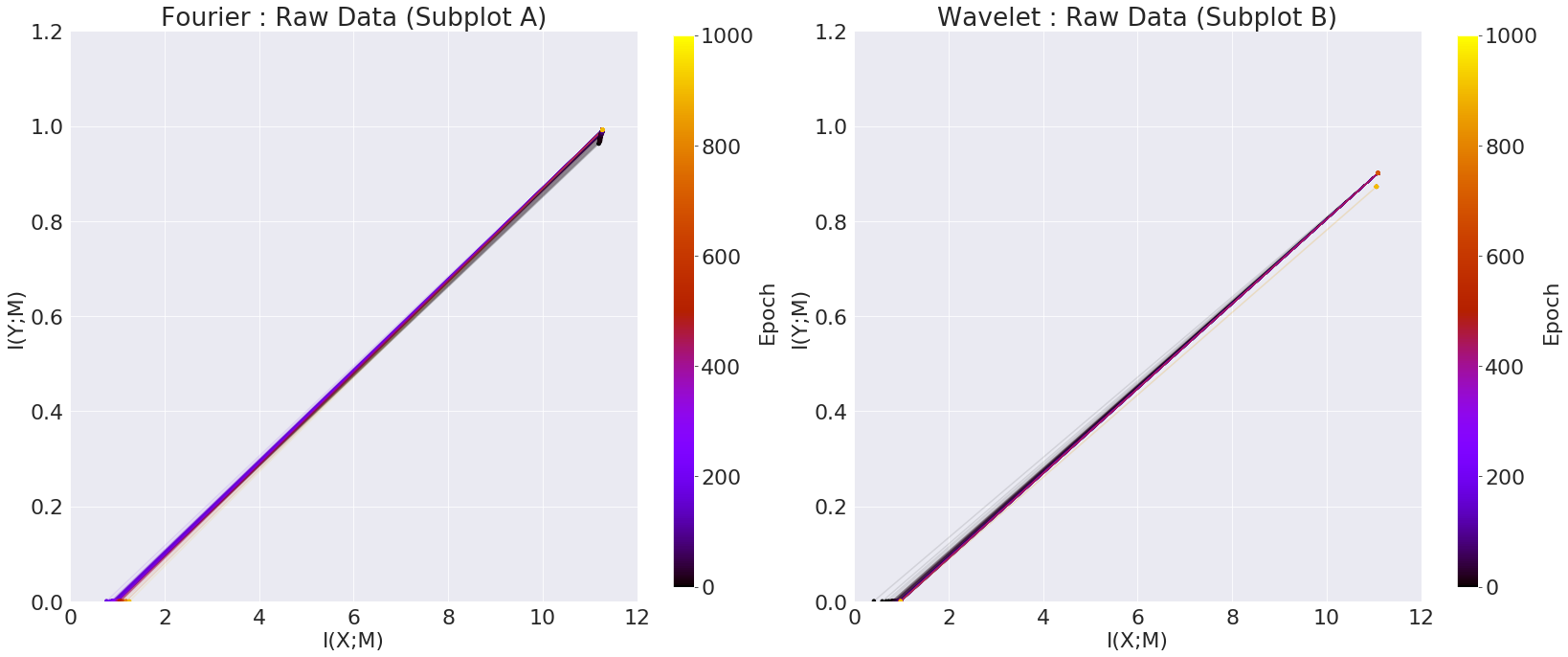}
\caption{Information Plane for Fourier and Wavelet neural networks for raw data}
\label{fig:malware infoplane fourier-wavelet}
\centering
\end{center}
\end{figure}

The similarity of the plots in Figure~\ref{fig:malware infoplane fc}, Figure~\ref{fig:malware infoplane conv}, and Figure~\ref{fig:malware infoplane fourier-wavelet} is not a coincidence, and the captured mutual information about the labels in the output layer is within a fraction of a bit for all of our networks. 
An unexpected deviation can be observed in Figure~\ref{fig:malware infoplane conv} subplot C, the information plane for the convolutional neural network trained on the wavelet data. 
This effect happens in the convolutional layers, and only on the wavelet-transformed data.
The cause of this change is unknown, and is reserved for future work.
Despite this difference, the dense hidden layers and the output layer converge to the same points in the information plane as the other networks and datasets.

In Figure~\ref{fig:malware infoplane fourier-wavelet}, we observe that the process of putting our data through the Fourier transform or Wavelet transform and the corresponding inverse transform, seems to preserve both mutual information and our accuracy.
Both subplots A and B mimic the information plane of the fully-connected neural networks and have nearly the same information plane graph as the convolutional neural network as seen in Figure~\ref{fig:malware infoplane conv}, Subplots A and B.

Some of the difference between the initial mutual information states for network and dataset combinations can be explained by the stochasticity in neural networks - that is, the weights of the network are initialized randomly, and samples are chosen at random. 
As a result, it is crucial to look at how the networks converge, and after 1000 epochs, the key factor between networks appears to be architecture - that is, whether the network is feedforward or convolutional - rather than the representation of the data itself.
Due to the remarkable similarity of the information plane, we conclude that the representations learned by the neural network are related more to latent structure in the data than to the specific values of the input data.

Further, for all of our neural networks, our information plane is quite similar, and converges to exactly the same value for mutual information in the densely connected hidden layers.
The exception is the summary dataset, which is the only dataset whose representation is not the result of a homeomorphic transformation.
We also note that this is the dataset with the worst evaluation accuracy.
%
\chapter{Experiment 2 - Dataset and in-Network Transformation of MNIST}
\label{chap:five}

\section{Motivation}
In order to contextualize the results of the experiments conducted on malware data, we consider the methods presented on the well-studied MNIST Database of Handwritten Digits~\cite{lecun1998mnist}.
MNIST is a benchmark in computer vision - since our baseline convolutional neural network is based on LeNet~\cite{lecun1998gradient}, we have a large body of research to compare to.
Additionally, MNIST serves as an introduction to the field of computer vision for many students and so our architectures and theories can be made more accessible in that context.
The current state of the art for MNIST achieved a 99.84 accuracy this year~\cite{byerly2020branching}.
The best results achieved in the original LeCun paper were 99.3\% accuracy; generally, accuracy greater than 97\% is considered to be good.

\section{Methodology}
In order to maintain consistency with our other findings, our methodology is the same as experiments in Chapter~\ref{chap:three}.
We leverage the hardware described in Appendix~\ref{append:one} and perform two sub-experiments.
The goal of the first part of our experiment is to best fit the data while ensuring generalization, and our goal is to optimize accuracy on the test set.
The results of this experiment are contained in Table~\ref{Tab:test}.
In the second part of our experiment, we ran the same mutual information computation as described in Section~\ref{MI computation}, and those results are displayed in Figure~\ref{fig:mnist fc infoplane} and Figure~\ref{fig:mnist conv infoplane}.

\section{Results}
In this experiment, all of our models achieve accuracy over 97\% which is broadly considered to be the benchmark accuracy for ``passing'' MNIST. 
For all these combinations of architecture and transformation, the difference between our maximum accuracy score of 99.11\% and our minimum accuracy of 97.73\% is only 1.38\%.
Comparing the convolutional models in particular, the difference between the Raw, Fourier-transformed, and Wavelet-transformed data is only 0.01\%. 

\begin{table}[h!]
\centering	
\begin{tabular}{l|ll}
\textbf{Data and Architecture}  & \textbf{Test Accuracy} & \textbf{Mean Step Time} ($\mu$s) \\\cline{1-3}
Raw, Fully-Connected NN            & 97.73\%         & 29\\
Fourier, Fully-Connected NN        & 98.12\%         & 41\\
Wavelet, Fully-Connected NN        & 97.83\%         & 30\\
\hline
Raw, Convolutional NN              & 99.11\%         & 204\\ 
Fourier, Convolutional NN          & 99.10\%         & 237\\
Wavelet, Convolutional NN          & 99.10\%         & 212\\
\hline
Raw, Fourier NN                    & 98.45\%         & 959\\
Raw, Wavelet NN                    & 98.89\%         & 1068\\ 
\end{tabular}
\caption{Neural Network Results}
\label{Tab:test}
\end{table}

\begin{figure}[h]
\begin{center}
\includegraphics[width=\textwidth]{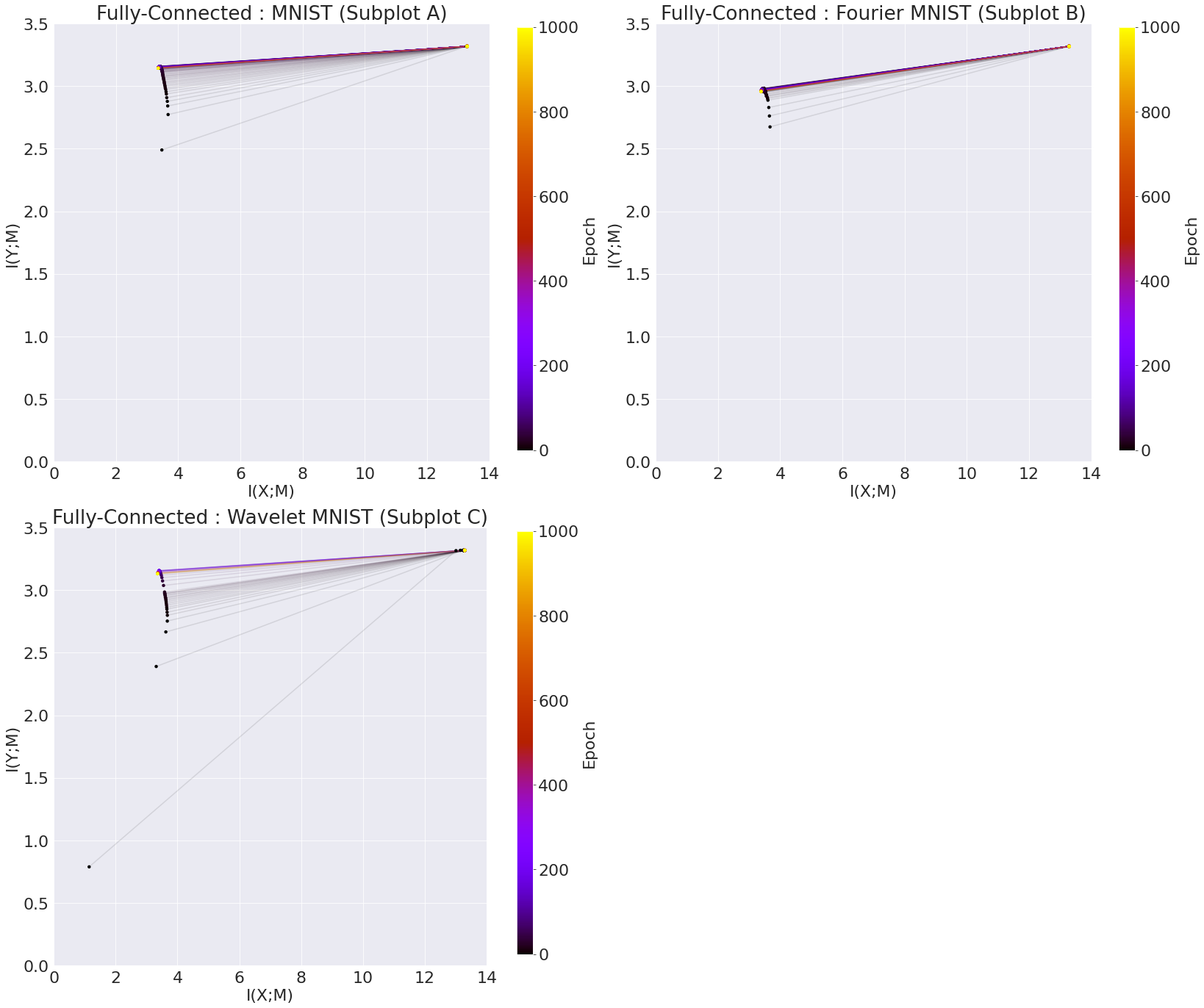}
\caption{Fully-Connected Neural Network Information Plane}
\label{fig:mnist fc infoplane}
\centering
\end{center}
\end{figure}
Figure~\ref{fig:mnist fc infoplane} was produced using the upper bound methodology from Saxe~\cite{saxe2019information}. 
We see from all three subplots in Figure ~\ref{fig:mnist fc infoplane} that over 1000 training epochs, the mutual information about the labels for the first and last layers of the neural net is quite similar for our fully-connected neural networks irrespective of the initial data.
We note that in Subplot C, the wavelet data yields a much lower amount of information about Y in the first training epoch, but quickly converges to the same point as the other two datasets, represented in subplots A and B.
As seen in both Shwartz-Ziv~\cite{shwartz2017opening} and Saxe~\cite{saxe2019information}, the early training epochs cause the largest increase in mutual information with respect to the labels, decreasing over training. 
We also observe a slight decrease in the amount of mutual information with respect to the data in later training epochs - which is expected as the learned representation becomes better able to map data to labels. 
The fully connected network converges to an upper bound for the output layer which is within half a bit across all three datasets - raw, Fourier-transformed, and Wavelet-transformed - within the first 500 epochs and begin to reduce their mutual information about $X$ as the network overfits the dataset.
Since we use ReLU activation functions, we do not see a ``fitting phase and compression phase'' as observed in Tishby~\cite{tishby2015deep}, but instead a simultaneous ``fitting and compression'' as in Saxe.

\begin{figure}[h!]
\begin{center}
\includegraphics[width=\textwidth]{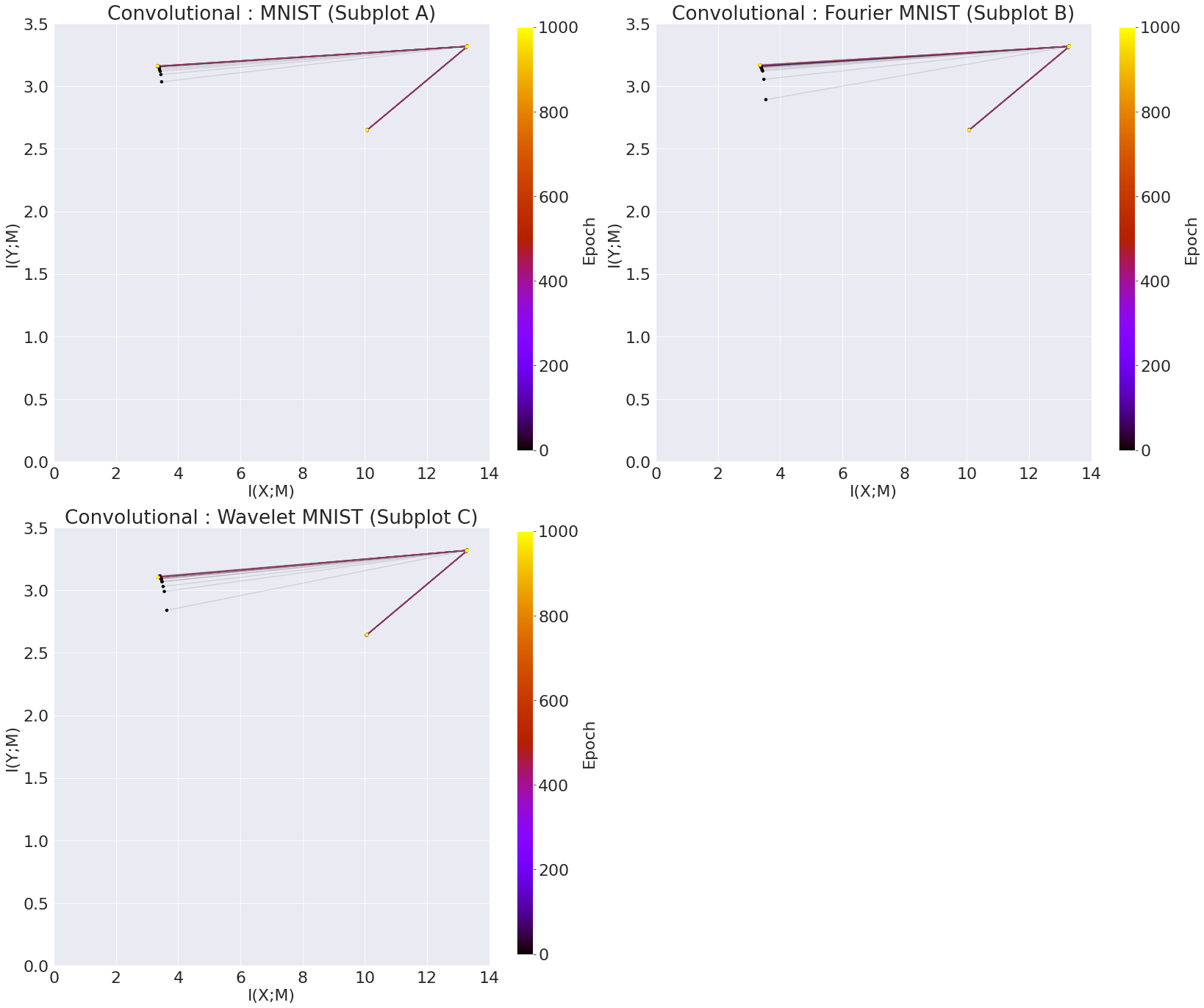}
\caption{Convolutional Neural Network Information Plane}
\label{fig:mnist conv infoplane}
\centering
\end{center}
\end{figure}

We observe very similar results in Figure~\ref{fig:mnist conv infoplane} with the convolutional network, and we can see that the convolutional layers, which for all 3 data representations, sits near the point $(10, 2.6)$ throughout training, do not budge from that point. 
Meanwhile, the densely connected layers sit at the same point in the information plane as in Figure~\ref{fig:mnist fc infoplane}.
We can also see that the output layer on the left moves to nearly the same point for all 3 models.
%
%
\chapter{Conclusions and Further Work}
\label{chap:conclusion}

\section{What a Neural Network Learns}
In Section~\ref{MNI_IB}, we show that that the representation learned by the neural network, $Z$ must constitute a minimum sufficient statistic of $X$ in order for the network to be predictive with respect to the information bottleneck optimality equation. 
Moreover, we demonstrate that the mutual information satisfies the data processing inequality with respect to the Markov chain $Y \to X \to Z$: $I(Y; X) \geq I(Y; Z)$.
Since the invariance of mutual information under homeomorphism allows us to affirm that any smooth, uniquely invertible map on $X$ does not impact the ability of a network to learn a representation, we have demonstrated that only methods of feature extraction~\cite{goodfellow2016deep} which change the data in ways that meaningfully change the entropy of $X$ are useful for altering the prediction accuracy of the network.
We know from natural language processing that some types of feature extraction which are not invertible improve the accuracy of prediction~\cite{goodfellow2016deep}, but given the results of our summary statistic data in Section~\ref{malware results}, not all feature extraction methods are equally valid.
This also makes intuitive sense from the standpoint of the data processing inequality outlined in Section~\ref{MNI_IB}, Equation~\ref{eqn:DPI}.
This strengthens the theory~\cite{krizhevsky2012imagenet} that probability mass is concentrated in locally-connected regions approximated by small manifolds with significantly lower dimensionality than $X$ itself, since these submanifolds would be preserved under this transformation.

Based on the results of our experiments in Chapter~\ref{chap:three} and Chapter~\ref{chap:five}, we observe that transformation of the dataset under homeomorphism has very little impact on the information plane.
Since the information plane is largely unchanged, and our accuracy remains quite similar, we conjecture that a smooth, bijective map applied to a dataset does not impact the ability of a neural network to learn a representation.
A rigorous proof of this conjecture is reserved for future work.

We used the convolution theorem to process smaller datasets than in Pratt~\cite{pratt2017fcnn} and Fujieda~\cite{fujieda2017wavelet} and found no loss of mutual information or accuracy.
However, we do not observe the speed increases in the previous work, possibly due to the disparity in our data size - the overhead of the transform and the inverse transform is larger than the improved speed of dot product over convolution.
This suggests that leveraging the convolution theorem to reduce computational load on large datasets may be worthwhile since we improve the speed of computation with no loss of information but is inefficient on smaller datasets.

\section{Malware Data Experiments}
In our malware data experiments, no neural network was able to match or surpass the accuracy of the random forest.
Additionally, the random forest is a model that is more interpretable, and trains much more quickly - two features that are highly desirable in information security.
No optimization was done on the hyperparameters of the decision tree, and so it is likely that a decision tree trained on raw data could achieve even higher accuracy results than were achieved in Chapter~\ref{chap:three}. 
Since each observation in our data is independent of the observation before it, the relationships are not complex and so it is plausible that a decision tree-based model could be architecturally optimal for our problem.

Our summary statistic dataset provided the most data interpretability from a human standpoint, and per model, provided the worst results. 
This result demonstrates that human interpretability of the data does not necessarily enhance the ability of neural networks to learn, even in relatively low dimensional spaces.
Some opportunity exists to enhance neural network-based detection, but this would likely require significantly larger volumes of data and more homogeneity between samples.
Further work could also be done to do manual feature extraction or additional correlation of metadata to improve detection rates.

\section{Further Work}
Our experiments contextualize results from experiments on the EMBER dataset performed by Anderson, Raff, and previous work done by the author of this paper~\cite{anderson2018ember, raff2018malware, galinkin2019shape}.
Anderson found that features extracted by experts with some light preprocessing outperformed featureless end-to-end deep learning even in spite of the ``natural'' feature extraction found in convolutional neural networks~\cite{he2016deep}.
Our previous work found that raw bytes are generally not a robust feature for malware detection, even if the support of the convolutional filter is considered and the filter shape is optimized for the target.
Our results here suggest that there may be some relevant change to the entropy when the executable is parsed as in Anderson's work.
This research serves as an avenue for future work.


There are implications of taking a manifold view in the space of adversarial examples~\cite{szegedy2013intriguing} that could allow us to minimize the dimension of the manifold and the order of the coordinate system, smoothing the loss surface, reducing the efficacy of gradient-based attacks~\cite{athalye2018obfuscated}.
This application has valuable contributions to the defense of machine learning systems, a threat that organizations are not prepared for~\cite{kumar2020adversarial}.
By using the ideas of a projection onto a manifold, we can categorize networks and datasets that might prove susceptible to adversarial examples.
Additionally, since we seek to minimize the information in our learned representation, model inversion attacks~\cite{zhang2019secret} become more challenging.

Our neural network's poor accuracy and the success of the random forest classifier also provide a potential avenue for further study.
If our random forest classifier learns rules that partition the dataset, rather than a function that maps inputs to labels, it may be architecturally optimal. 
This would require a multi-dimensional analysis of the data and examining in-depth the branching points of the random forest classifier.
Though this work is outside the scope of this thesis, it would provide an insight into when and why to choose certain machine learning models given properties of the dataset.

Finally, when plotting the information plane, the parameter that seemed to have the greatest effect on the magnitude of the changes was the number of neurons, especially in the case of a feed-forward neural network.
We observed that using very small numbers of neurons by modern neural network sizes: 4 to 16 neurons per layer, for example, we saw much lower initial levels of mutual information, which would still eventually converge to the same points. 
We did not explore why this is the case, and reserve investigation of this phenomenon to future work.
%
%
%
%
\bibliographystyle{IEEEtran}
%
\setbibref{References}
\bibliography{bibtex}
%
\begin{appendices}
\footnotesize
\addtocontents{toc}{\protect\renewcommand\protect\cftchappresnum{\appendixname~}}
\renewcommand{\thechapter}{\Roman{chapter}}

\renewcommand{\thesection}{\Alph{section}.}

\chapter{Hardware}
\label{append:one}

All models were trained on the same hardware with the following specifications:

\quad\textbf{CPU}: AMD Ryzen Threadripper 2920x 12-core 3.5 GHz

\quad\textbf{RAM}: 128 GB 3200 MHz DDR4

\quad\textbf{GPU}: Nvidia RTX 2080 Ti 12GB

All models for which the software was compatible with GPU were trained on GPU.
Datasets were all small enough to be held in memory after having been read from disk so disk i/o latency was not a factor in training times.
\end{appendices}
%
\chap{Biographical sketch}
\normalsize
\nonfrenchspacing
\doublespacing
\setlength{\parindent}{15pt} 
\setlength{\parskip}{3pt} 

Erick Galinkin was born in Port Jefferson, New York just in time for coverage of Tiananmen square to be televised.
After graduating from Miller Place High School in 2007, he spent a brief time at Stony Brook University where he studied Chemistry and conducted research into electrical properties of organometallics under Dr. Fernando O. Raineri. 
At the end of 2008, he joined the United States Air Force as a Chinese Language Analyst and spent 6 years enlisted, obtaining an Associate's Degree in Mandarin Chinese from the Defense Language Institute and a Bachelor's Degree in Cybersecurity from the University of Maryland Global Campus during this time.
More importantly, Erick's first son, Noble was born in December of 2010 and his middle son, Jameson was born in September of 2012.
Upon separating from the Air Force in 2014, Erick began working as a Research Engineer for Cisco Systems and shortly thereafter began pursuing a masters in information assurance and a graduate certificate in Bioinformatics, graduating in 2017 just before the birth of his youngest son, Beckett.
Beginning in September 2018, Erick enrolled at Johns Hopkins and started work as a Security Research Scientist at Netskope.
In January, 2020 he began work as a Principal Artificial Intelligence Researcher at Rapid7 and has accepted an offer to further his studies as a PhD student in Computer Science at Drexel University.
\end{document}